\documentclass{article}

\PassOptionsToPackage{numbers, compress}{natbib}
\usepackage[final]{nips_2018} 

\usepackage[utf8]{inputenc} 
\usepackage[T1]{fontenc}    
\usepackage{hyperref}       
\usepackage{url}            
\usepackage{booktabs}       
\usepackage{amsfonts}       
\usepackage{nicefrac}       
\usepackage{microtype}      
\usepackage{bm}
\usepackage{amsmath}
\usepackage{amsthm}
\usepackage{graphicx}
\usepackage{epstopdf}
\usepackage{caption}
\usepackage{subfigure}
\usepackage{enumerate}
\usepackage{float}
\usepackage{algorithmic}
\usepackage{algorithm}
\usepackage{float}
\usepackage{picins}

\newtheorem{lemma}{Lemma}

\newtheorem{definition}{Definition}

\title{Transform-Based Multilinear Dynamical System \\for Tensor Time Series Analysis}

\author{Weijun Lu$^+$,
       Xiao-Yang Liu$^*$,
       Qingwei Wu$^*$,
       Yue Sun$^+$,
       and Anwar Walid$^\dagger$\\
       $^+$State Key Lab. of Integrated Service Networks, Xidian University, \\
       $^*$Electrical Engineering, Columbia University,\\
       $^\dagger$Mathematics of Systems Research Department, Nokia-Bell Labs\\
       Emails: 623116319@qq.com, \{XL2427, QW2208\}@columbia.edu, \\ sunyue@xidian.edu.cn, anwar.walid@nokia-bell-labs.com
       }

\begin{document}

\maketitle
\begin{abstract}

We propose a novel multilinear dynamical system (MLDS) in a transform domain, named $\mathcal{L}$-MLDS, to model tensor time series. With transformations applied to a tensor data, the latent multidimensional correlations among the frontal slices are built, and thus resulting in the computational independence in the transform domain. This allows the exact separability of the multi-dimensional problem into multiple smaller LDS  problems. To estimate the system parameters, we utilize the expectation-maximization (EM) algorithm to determine the parameters of each LDS. Further, $\mathcal{L}$-MLDSs significantly reduce the model parameters and allows parallel processing. Our general $\mathcal{L}$-MLDS model is implemented based on different transforms: discrete Fourier transform, discrete cosine transform and discrete wavelet transform. Due to the nonlinearity of these transformations, $\mathcal{L}$-MLDS is able to capture the nonlinear correlations within the data unlike the MLDS \cite{rogers2013multilinear} which assumes multi-way linear correlations. Using four real datasets, the proposed $\mathcal{L}$-MLDS is shown to achieve much higher prediction accuracy than the state-of-the-art MLDS and LDS with an equal number of parameters under different noise models. In particular, the relative errors are reduced by $50\% \sim 99\%$. Simultaneously, $\mathcal{L}$-MLDS achieves an exponential improvement in the model's training time than MLDS.

\end{abstract}

\section{Introduction}
Predicting the evolving trends of data sequences is an essential problem arising in various fields such as signal processing, environmental protection and economics. A traditional model to describe a dynamically evolving data sequence is the linear dynamical system (LDS), where the observations and latent states are expressed as vectors. In the era of big data, data in various applications is frequently represented as a time series of multidimensional arrays, called tensors, to preserve the inherent multidimensional correlations. \textit{Of interest is the prediction of future terms of the time tensor series.} The obvious solution is to unfold each tensor to a vector, and then the LDS applies as in \cite{bishop2006pattern,ghahramani1996parameter}. LDS can not preserve the data structure, and it does not allow determination of the dimension of each mode of the latent tensor. Bayesian probabilistic tensor factorization (BPTF) \cite{xiong2010temporal} is an explicit model for predicting tensor time series, which concatenates the members of the tensor time series and yields a higher-order tensor. Though BPTF preserves the tensorial structure, the latent structure is limited.

A multilinear dynamical system (MLDS) for modeling time tensor series is proposed in \cite{rogers2013multilinear} to generalize the LDS by vectorizing the input tensors. Expressing the latent states and observations as tensors and replacing the transition and projection matrices with multilinear operators, MLDS preserves the tensorial structure of the data. The multilinear operators are factorized as the Kronecker product of multiple smaller matrices so that the number of model parameters is significantly reduced (compared to LDS). MLDSs preserve the tensor structure with more flexible dimensionalities of the latent tensors and achieve a higher prediction accuracy than LDS. However, MLDSs still take a high computational cost to estimate the large number of covariance parameters. Moreover, the methods for estimating the multilinear operators of MLDS in \cite{rogers2013multilinear} may fall into local optimum, thus compromising the prediction accuracy.

To address the aforementioned issues, we propose a novel multilinear dynamical system based on transform-based tensor model (in which the transform is denoted by $\mathcal{L}$, thus we call it $\mathcal{L}$-MLDS). Working in the transform domain, we define a probabilistic model to construct the $\mathcal{L}$-MLDS. The multilinear operators and covariances of the $\mathcal{L}$-MLDS model are represented as sparse block diagonal matrices in the transform domain, allowing exact separability of the $\mathcal{L}$-MLDS model into multiple smaller LDSs and providing the opportunity for parallel processing. To estimate the model parameters, we utilize the standard EM algorithm to determine the parameters of each LDS in the transform domain. Therefore, the model involves fewer parameters, simple estimation procedures, efficient computation, and potential for parallel processing, leading to improvements in the model training.

The $\mathcal{L}$-MLDS model allows arbitrary noise relationships among the tensorial elements without the restrictive assumption of isotropic noise used in \cite{surana2016dynamic,sun2006beyond}. To assess the performance of the $\mathcal{L}$-MLDS model, we implement it using three different transformations and conduct experiments under different noise models. Simulation results with real data demonstrate that the proposed $\mathcal{L}$-MLDS achieves higher prediction accuracy than MLDS and LDS with an equal number of parameters while taking much less time to train the model.

\section{Transform-Based Tensor Model}\label{sec2}
Let $\mathbb{C}$ denote complex numbers. Vectors are denoted by boldface lowercase letters, e.g., $\bm{a}$; matrices are denoted by boldface capital letters, e.g., $\bm{A}$; and higher-order tensors are denoted by calligraphic letters, e.g., $\mathcal{A}$. The index set $\{1, 2,\cdots, n\}$ is denoted by [$n$]. Let $\mathbb{N}$ denote positive integers and $I, J, K\in \mathbb{N}$, for a third-order tensor $\mathcal{A}\in \mathbb{C}^{I\times J\times K}$, we use $\mathcal{A}(i, j, :)$ to denote the mode-3 tube and $\mathcal{A}^{(k)}$ to denote the $k$-th frontal slice \cite{kolda2009tensor}. In this paper, we just consider the third-order tensor for ease of exposition.

\textbf{Basic operators} \cite{liu2017fourth}: The operator MatView($\cdot$) takes a tensor $\mathcal{A}\in \mathbb{C}^{I\times J\times K}$ and returns an $IK\times JK$ block diagonal matrix, with each block being an $I\times J$ matrix, defined as
\begin{equation}\label{mat}
  \text{MatView}(\mathcal{A})=\text{diag}(\mathcal{A}^{(1)}, \cdots, \mathcal{A}^{(k)}, \cdots, \mathcal{A}^{(K)}).
\end{equation}
The operator Vec($\cdot$) takes a tensor $\mathcal{B}\in \mathbb{C}^{I\times 1\times K}$ and returns a vector of length $IK$, defined as
\begin{equation}
\text{Vec}(\mathcal{B})=[\mathcal{B}^{(1)};\cdots; \mathcal{B}^{(k)}; \cdots; \mathcal{B}^{(K)}].
\end{equation}
Conversely, the operator TenView($\cdot$) folds MatView($\mathcal{A}$) and Vec($\mathcal{B}$) back to tensors $\mathcal{A}$ and $\mathcal{B}$, respectively, i.e., $\text{TenView}(\text{MatView}(\mathcal{A}))=\mathcal{A}$ and $\text{TenView}(\text{Vec}(\mathcal{B}))=\mathcal{B}$.

Given an invertible discrete transform $\mathcal{L}:\mathbb{C}^{K}\rightarrow\mathbb{C}^{K}$, the elementwise multiplication is denoted by $\circ$, and with $\bm{\alpha}, \bm{\beta}\in \mathbb{C}^{K}$, the \textit{tubal-scalar multiplication} $\bullet$ is defined \cite{liu2017fourth} as $\bm{\alpha}\bullet\bm{\beta}=\mathcal{L}^{-1}(\mathcal{L}(\bm{\alpha})\circ\mathcal{L}(\bm{\beta}))$, and $\mathcal{L}^{-1}$ is the inverse of $\mathcal{L}$.
\piccaption[]{Transformations taken along the third dimension.\label{transform}}
\parpic[r][r]{\includegraphics[width=.30\textwidth]{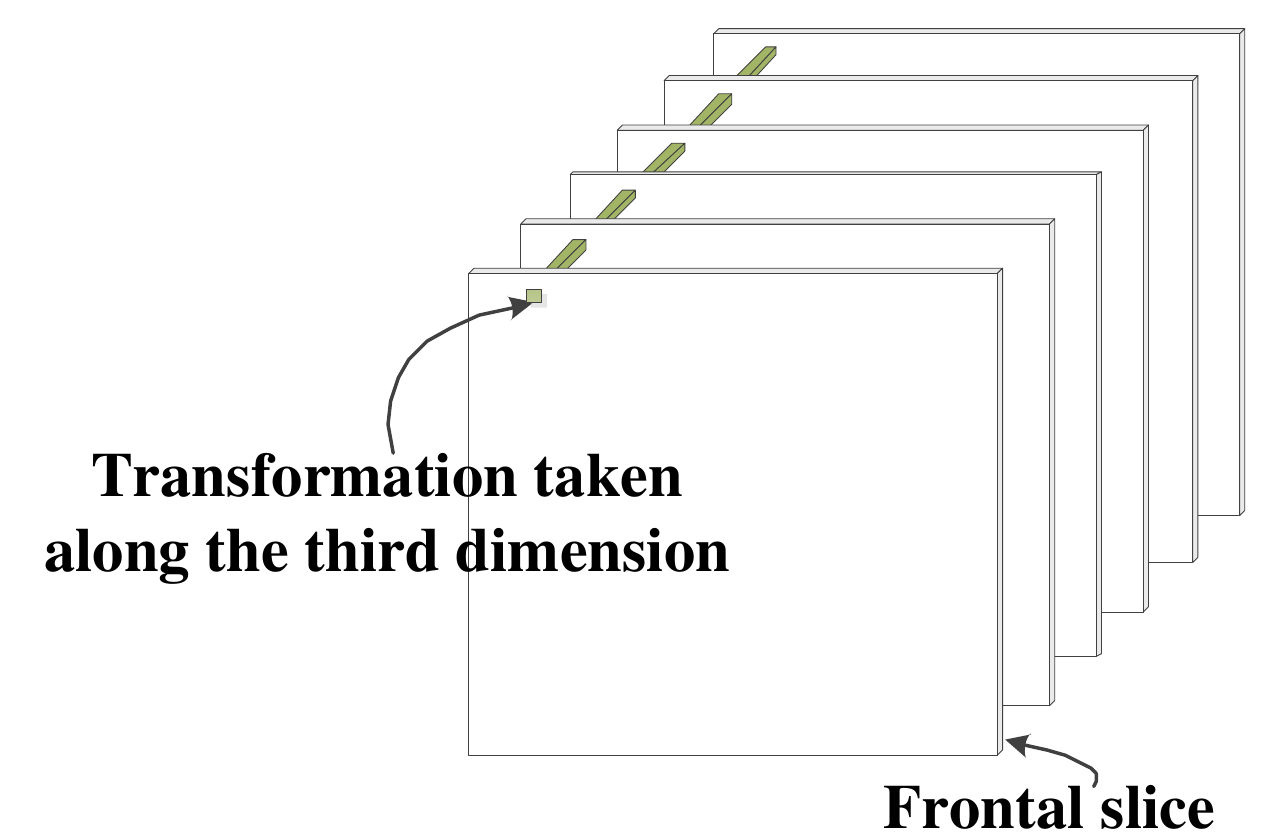}}

We use $\widetilde{\mathcal{X}}=\mathcal{L}(\mathcal{A})\in\mathbb{C}^{I\times J\times K}$ to denote the tensor obtained by taking the transform $\mathcal{L}$ of all the tubes along the third dimension of $\mathcal{A}\in \mathbb{C}^{I\times J\times K}$. The transformation $\mathcal{L}$ builds the correlations among the frontal slices in the transform domain just like threading the wires through them. Therefore, while the frontal slices of a tensor in time domain are dependent, they are in fact \emph{independent} in the transform domain.

\begin{definition}\label{lpro}
\cite{liu2017fourth}
The $\mathcal{L}$-product $\mathcal{C}=\mathcal{A}\bullet\mathcal{B}$ of $\mathcal{A}\in\mathbb{C}^{I\times K\times L}$ and $\mathcal{B}\in\mathbb{C}^{K\times J\times L}$ is a tensor in $\mathbb{C}^{I\times J\times L}$, with $\mathcal{C}(i,j,:)=\sum_{k=1}^K\mathcal{A}(i,k,:)\bullet\mathcal{B}(k,j,:)$, for $i\in[I]$ and $j\in[J]$.
\end{definition}

\begin{lemma}\label{lem1}
\cite{liu2017fourth}
The $\mathcal{L}$-product $\mathcal{C}=\mathcal{A}\bullet\mathcal{B}$ can be converted to the matrix multiplication in the transform domain, similar to the convolution theorem, $\text{MatView}(\widetilde{\mathcal{C}})=\text{MatView}(\widetilde{\mathcal{A}})\cdot \text{MatView}(\widetilde{\mathcal{B}})$.
\end{lemma}
In particular, given $\mathcal{A}\in\mathbb{C}^{I\times J\times K}$ and $\mathcal{B}\in\mathbb{C}^{J\times 1\times K}$, the $\mathcal{L}$-product $\mathcal{C}=\mathcal{A}\bullet\mathcal{B}$ can be calculated as $\text{Vec}(\widetilde{\mathcal{C}})=\text{MatView}(\widetilde{\mathcal{A}})\cdot \text{Vec}(\widetilde{\mathcal{B}})$.
In this case, we call $\mathcal{A}$ a multilinear operator of $\mathcal{B}$ \cite{liu2017fourth, kilmer2013third, XiaoYang2016Low}.

Different from the tensor normal distribution in \cite{basser2002normal} which is restricted to symmetric second-order tensors, we define an $\mathcal{L}$-normal distribution in a transform domain for arbitrary second-order tensors. The corresponding random tensors, called $\mathcal{L}$-random tensors, are used to construct $\mathcal{L}$-MLDS.
\begin{definition}\label{definition}
\textbf{($\mathcal{L}$-Normal Distribution)} Given a tensor $\mathcal{X} \in\mathbb{C}^{J\times 1\times K}$, let $\widetilde{\mathcal{X}}=\mathcal{L}(\mathcal{X})$, and then we say $\mathcal{X}$ has the $\mathcal{L}$-normal distribution with expectation $\mathcal{U}\in \mathbb{C}^{J\times 1\times K}$ and covariance $\mathcal{Q}\in \mathbb{C}^{J\times J\times K}$, denoted by
\begin{equation}\label{L-NorD}
 {\mathcal{X}} \sim \mathcal{CN}_{\mathcal{L}}({\mathcal{U}},\mathcal{Q}),
\end{equation}
if and only if
\begin{equation}\label{transcov}
\text{Vec}(\widetilde{\mathcal{X}}) \sim \mathcal{CN}(\text{Vec}(\widetilde{\mathcal{U}}),\text{MatView}(\widetilde{\mathcal{Q}})),
\end{equation}
where $\mathcal{CN}$ is the traditional complex normal distribution.
\end{definition}

Suppose $\mathcal{X}\in \mathbb{C}^{J\times 1 \times K}$ and $\mathcal{Y}\in \mathbb{C}^{I\times 1 \times K}$ are jointly distributed as
\begin{equation}\label{jointD}
  {\mathcal{X}} \sim \mathcal{CN}_{\mathcal{L}} ({\mathcal{U}}~,~\mathcal{Q})~~~~~~\text{and}~~~~~~ \mathcal{Y}\mid\mathcal{X}\sim\mathcal{CN}_{\mathcal{L}}(\mathcal{C}\bullet\mathcal{X},\mathcal{R}),
\end{equation}
where $\mathcal{C}\in \mathbb{C}^{I\times J \times K}$ is the multilinear operator of $\mathcal{X}$. We can obtain the marginal distribution of $\mathcal{Y}$ and the posterior distribution of $\mathcal{X}$ given $\mathcal{Y}$ as follows.

\begin{lemma}Suppose the joint distribution of $\mathcal{L}$-random tensors $\mathcal{Y}\in \mathbb{C}^{I\times 1 \times K}$ and $\mathcal{X}\in \mathbb{C}^{J\times 1 \times K}$ is given by (\ref{jointD}), then the marginal distribution of $\mathcal{Y}$ is
\begin{equation}\label{eqnY}
  \mathcal{Y}\sim\mathcal{CN}_{\mathcal{L}} (\mathcal{C}\bullet\ \mathcal{U},\mathcal{C}\bullet\mathcal{Q}\bullet\mathcal{C}^H+\mathcal{R}),
\end{equation}
where $\mathcal{C}^H\in\mathbb{C}^{J\times I \times K}$. The conditional distribution of $\mathcal{X}$ given $\mathcal{Y}$ is
\begin{equation}\label{eqnXY}
  \mathcal{X}\mid\mathcal{Y}\sim\mathcal{CN}_{\mathcal{L}} (\mathcal{M},\mathcal{G}),
\end{equation}
where
$\mathcal{M}=\mathcal{L}^{-1}(\text{TenView}(\bm{\Sigma}(\bm{\widetilde{C}}^H \bm{\widetilde{R}}^{-1} \text{Vec}(\widetilde{\mathcal{Y}})+\bm{\widetilde{Q}}^{-1} \text{Vec}(\widetilde{\mathcal{U}}))))$, $\mathcal{G}=\mathcal{L}^{-1}(\text{TenView}(\bm{\Sigma}))$,
$\bm{\widetilde{R}}=\text{MatView}(\widetilde{\mathcal{R}})$, $\bm{\widetilde{Q}}=\text{MatView}(\widetilde{\mathcal{Q}})$, $\bm{\widetilde{C}}=\text{MatView}(\widetilde{\mathcal{C}})$, and $\bm{\Sigma}=(\bm{\widetilde{Q}}^{-1}+\bm{\widetilde{C}}^H\bm{\widetilde{R}}^{-1} \bm{\widetilde{C}})^{-1}$. $\mathcal{C}^H\in\mathbb{C}^{J\times I\times K}$ such that $\text{MatView}(\mathcal{L}(\mathcal{C}^H))=\text{MatView}(\mathcal{L}(\mathcal{C}))^H$ \cite{liu2017fourth}.
\end{lemma}
\begin{proof} Definition \ref{definition}, Lemma \ref{lem1} and (\ref{jointD}) imply that Vec($\widetilde{\mathcal{X}}$) and Vec($\widetilde{\mathcal{Y}}$) given Vec($\widetilde{\mathcal{X}}$) in the transform domain follow:
\begin{eqnarray*}
  \text{Vec}(\widetilde{\mathcal{X}})& \sim& \mathcal{CN}(\text{Vec}(\widetilde{\mathcal{U}}),\text{MatView}(\widetilde{\mathcal{Q}})),\\
  \text{Vec}(\widetilde{\mathcal{Y}})~|~\text{Vec}(\widetilde{\mathcal{X}})&\sim &\mathcal{CN}(\text{MatView}(\widetilde{\mathcal{C}})\cdot \text{Vec}(\widetilde{\mathcal{X}}),\text{MatView}(\widetilde{\mathcal{R}}) ).
\end{eqnarray*}
By the properties of the multivariate normal distribution \cite{bishop2006pattern}, the marginal distribution of $\text{Vec}(\widetilde{\mathcal{Y}})$ and the conditional distribution of $\text{Vec}(\widetilde{\mathcal{X}})$ given $\text{Vec}(\widetilde{\mathcal{Y}})$ are
\begin{eqnarray}\label{marginY1}
\text{Vec}(\widetilde{\mathcal{Y}}) &\sim&  \mathcal{CN}(\bm{\widetilde{C}}\cdot\text{Vec}(\widetilde{\mathcal{U}}),\bm{\widetilde{C}} \bm{\widetilde{Q}} \bm{\widetilde{C}}^H+\bm{\widetilde{R}}), \\
\label{marginY2}
\text{Vec}(\widetilde{\mathcal{X}})~|~\text{Vec}(\widetilde{\mathcal{Y}})&\sim&  \mathcal{CN}(\bm{\Sigma}(\bm{\widetilde{C}}^H \bm{\widetilde{R}}^{-1} \text{Vec}(\widetilde{\mathcal{Y}})+\bm{\widetilde{Q}}^{-1} \text{Vec}(\widetilde{\mathcal{U}})),\bm{\Sigma}).
\end{eqnarray}
Converting (\ref{marginY1}) and (\ref{marginY2}) back to the tensor forms by Definition \ref{definition}, we can obtain (\ref{eqnY}) and (\ref{eqnXY}), respectively.
\end{proof}

\section{Transform-Based Multilinear Dynamical System}\label{sec4}
We define the $\mathcal{L}$-MLDS by treating each tensor $\mathcal{Y}_n$ as an $\mathcal{L}$-random tensor and relating each model component with a multilinear transformation.
\subsection{System Description}
The $\mathcal{L}$-MLDS model consists of a sequence $\mathcal{X}_1, \cdots,\mathcal{X}_N$ of latent tensors, where $\mathcal{X}_n\in \mathbb{C}^{J\times 1 \times K}$ for all $n$. Each latent tensor $\mathcal{X}_n$ associates with an observation $\mathcal{Y}_n\in \mathbb{C}^{I\times 1 \times K}$. The $\mathcal{L}$-MLDS is initialized by a latent tensor $\mathcal{X}_1$ distributed as
\begin{equation}\label{model1}
  {\mathcal{X}_1} \sim \mathcal{CN}_{\mathcal{L}} ({\mathcal{U}_0},\mathcal{Q}_0).
\end{equation}
Given $\mathcal{X}_n$, $1\leq n\leq N-1$, we generate $\mathcal{X}_{n+1}$ according to the conditional distribution
\begin{equation}\label{model2}
  \mathcal{X}_{n+1}\mid\mathcal{X}_{n}\sim \mathcal{CN}_{\mathcal{L}} (\mathcal{A}\bullet\mathcal{X}_{n},\mathcal{Q}),
\end{equation}
where $\mathcal{Q}$ is the conditional covariance tensor shared by all $\mathcal{X}_{n}, 2\leq n\leq N$, and $\mathcal{A}\in \mathbb{C}^{J\times J \times K}$ is the transition tensor which describes the dynamics of the evolving sequence $\mathcal{X}_{1}, \cdots , \mathcal{X}_{N}$. For each $\mathcal{X}_{n}$, the corresponding observation $\mathcal{Y}_{n}$ is generated by the conditional distribution
\begin{equation}\label{model3}
  \mathcal{Y}_{n}\mid\mathcal{X}_{n}\sim\mathcal{CN}_{\mathcal{L}} (\mathcal{C}\bullet\mathcal{X}_{n},\mathcal{R}),
\end{equation}
where $\mathcal{R}$ is the conditional covariance tensor shared by all $\mathcal{Y}_{n}$, and $\mathcal{C}\in \mathbb{C}^{I\times J \times K}$ is the projection tensor which transforms latent $\mathcal{X}_{n}$ to the corresponding observation $\mathcal{Y}_{n}$.

\subsection{Parameter Complexity Analysis}
Suppose the relationships among the elements of a tensor are non-independent. Then the number of parameters in LDS is
\begin{equation}\label{NV}
N_{\text{Para}}^{\text{LDS}}=|\mathcal{A}|+|\mathcal{C}|+(|\mathcal{Q}_0|+|\mathcal{Q}|)+|\mathcal{R}|=(JK)^2+IJK^2+2(JK)^2+(IK)^2.
\end{equation}
where $|\mathcal{M}|$ stands for the number of parameters of $\mathcal{M}$. For the MLDS model, the multilinear operators and covariance tensors of an second-order tensor are fourth-order tensors \cite{rogers2013multilinear}, thus
\begin{equation}
N_{\text{Para}}^{\text{MLDS}}=|\mathcal{A}|+|\mathcal{C}|+(|\mathcal{Q}_0|+|\mathcal{Q}|)+|\mathcal{R}|=IJ+J^2+2K^2+(IK)^2+2(JK)^2.
\end{equation}
While for the $\mathcal{L}$-MLDS, the multilinear operators and the covariances are sparse in the transform domain, i.e.,
\begin{equation}\label{NL}
N_{\text{Para}}^{\mathcal{L}\text{-MLDS}}=|\mathcal{A}|+|\mathcal{C}|+(|\mathcal{Q}_0|+|\mathcal{Q}|)+|\mathcal{R}|=J^2K+IJK+2J^2K+I^2K.
\end{equation}
Suppose $I=J=K=n$, then the parameter complexities of LDS and MLDS are ${O}(n^4)$ and that of $\mathcal{L}$-MLDS is ${O}(n^3)$. Thus $\mathcal{L}$-MLDS significantly reduces the number of parameters as the dimensions of the tensors increase. Conversely, with equal number of parameters, $\mathcal{L}$-MLDS tends to have a greater dimensionality ($J\times K$) of the latent state. Generally, the longer the vectorized latent tensor is, the more information of the corresponding observation it has. Therefore, $\mathcal{L}$-MLDS is able to achieve higher prediction accuracy than LDS and MLDS.
\subsection{System Identification}

The problem of $\mathcal{L}$-MLDS identification is to estimate the parameters $\Theta=\{\mathcal{U}_0,\mathcal{Q}_0,\mathcal{A},\mathcal{Q},\mathcal{C},\mathcal{R}\}$ from the given time series of observations $\mathcal{Y}_1,\cdots,\mathcal{Y}_N$. For the existence of unknown latent states $\mathcal{X}_n$ in the $\mathcal{L}$-MLDS, we cannot directly maximize the likelihood of the data with respect to $\Theta$.

According to Definition \ref{definition}, the $\mathcal{L}$-MLDS specified by (\ref{model1}), (\ref{model2}), and (\ref{model3}) can be divided into $K$ independent LDSs in the transform domain with each LDS being defined as
\begin{eqnarray}
~~~\left\{
   \begin{array}{ll}
~~~~~~~~~~\widetilde{\mathcal{X}}_1^{(k)}\sim \mathcal{CN}(\widetilde{\mathcal{U}}_0^{(k)},\widetilde{\mathcal{Q}}_0^{(k)}),\\
\widetilde{\mathcal{X}}_{n+1}^{(k)}|\widetilde{\mathcal{X}}_n^{(k)}\sim  \mathcal{CN}(\widetilde{\mathcal{A}}^{(k)}\cdot\widetilde{\mathcal{X}}_n^{(k)},\widetilde{\mathcal{Q}}^{(k)}),\\
~~\widetilde{\mathcal{Y}}_n^{(k)}|\widetilde{\mathcal{X}}_n^{(k)}\sim \mathcal{CN}(\widetilde{\mathcal{C}}^{(k)}\cdot\widetilde{\mathcal{X}}_n^{(k)},\widetilde{\mathcal{R}}^{(k)}).
    \end{array}
\right.
\end{eqnarray}
Hence, the problem of estimating $\Theta$ is exactly separated into $K$ independent subproblems of estimating $\theta^{(k)}=\{\widetilde{\mathcal{U}}_0^{(k)},\widetilde{\mathcal{Q}}_0^{(k)}, \widetilde{\mathcal{A}}^{(k)}, \widetilde{\mathcal{Q}}^{(k)}, \widetilde{\mathcal{C}}^{(k)}, \widetilde{\mathcal{R}}^{(k)}\}$ with incomplete data \cite{dempster1977maximum}. Then, we use the EM algorithm to estimate each $\theta^{(k)}$, $k\in[K]$, and finally convert all those subsystem components to time domain. For the specific process, see Figure \ref{train}.
\begin{figure}[t]
  \centering
  \includegraphics[width=12cm]{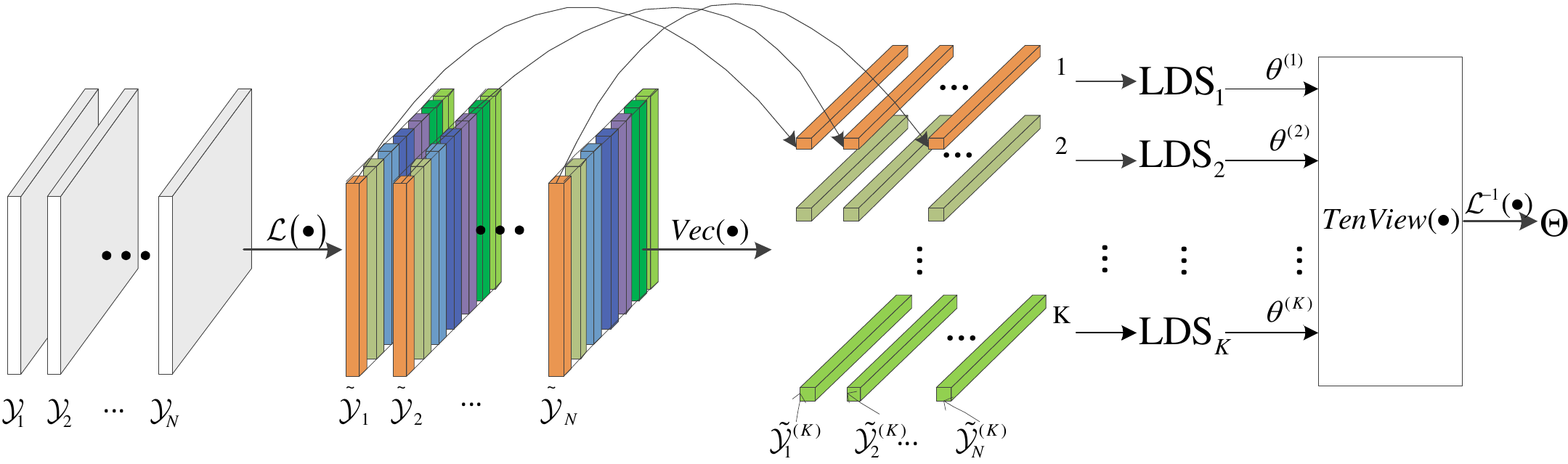}\\
  \caption{The process of $\mathcal{L}$-MLDS training.}\label{train}
  \vspace{-0.25in}
\end{figure}
\section{Performance Evaluation}\label{sec5}
\subsection{Model Initializations}
We evaluate the performance of $\mathcal{L}$-MLDS by comparing it with MLDS \cite{rogers2013multilinear} and LDS (vectorized tensors as inputs) on real data. Our general $\mathcal{L}$-MLDS model is implemented based on discrete Fourier transform (dft-MLDS), discrete cosine transform (dct-MLDS) and discrete wavelet transform (dwt-MLDS). The $\mathcal{L}$-MLDS parameters are initialized such that $\widetilde{\mathcal{U}}_0^{(k)}$ is drawn from the standard normal distribution, the diagonal block matrices $\widetilde{\mathcal{Q}}_0^{(k)}$, $\widetilde{\mathcal{Q}}^{(k)}$ and $\widetilde{\mathcal{R}}^{(k)}$ are identity matrices for each $k$, and the columns of $\widetilde{\mathcal{A}}^{(k)}$ and $\widetilde{\mathcal{C}}^{(k)}$ are the first $J$ eigenvectors of singular-value-decomposed matrices with entries drawn from the standard normal distribution. The LDS parameters are initialized in the same way as the $k$-th LDS of $\mathcal{L}$-MLDS. The MLDS parameters are initialized in the same way as \cite{rogers2013multilinear}.

Denote $\| \mathcal{X}\|_F=\sqrt{\sum_{i,j,k}|\mathcal{X}_{ijk}|^2}$ as the Frobenius norm of a third-order tensor. The prediction error $\varepsilon_n^{\mathcal{L}}$ of a given $\mathcal{L}$-MLDS model for the $n$-th member $\mathcal{Y}_n$ in a time series is
\begin{equation}
\varepsilon_n^{\mathcal{L}}=\frac{\| \mathcal{Y}_n^{\mathcal{L}}-\mathcal{Y}_n \|_F}{\|\mathcal{Y}_n\|_F}.
\end{equation}
Let $\text{E}[\mathcal{X}]$ denote the expectation of $\mathcal{X}$. Each estimate $\mathcal{Y}_n^{\mathcal{L}}$ of $\mathcal{L}$-MLDS is given in the following way:
we compute the prediction $\widehat{\mathcal{Y}}_{n}^{(k)}$ of the $k$-th LDS in the transform domain, i.e.,
\begin{equation}
\widehat{\mathcal{Y}}_{n}^{(k)}=\widetilde{\mathcal{C}}^{(k)}\cdot(\widetilde{\mathcal{A}}^{(k)})^n\cdot \text{E}[{\widehat{\mathcal{X}}}_{N_{\text{train}}}^{(k)}],
\end{equation}
where $\widehat{\mathcal{X}}_{N_{\text{train}}}^{(k)}$ is the estimate of latent state of the last member of the training sequence in the $k$-th LDS.
Then, $\text{Vec}(\widetilde{\mathcal{Y}}_n^{\mathcal{L}})=[\widehat{\mathcal{Y}}_{n}^{(1)}; \cdots; \widehat{\mathcal{Y}}_{n}^{(k)}; \cdots; \widehat{\mathcal{Y}}_{n}^{(K)}]$, we obtain $\mathcal{Y}_n^{\mathcal{L}}=\mathcal{L}^{-1}(\text{TenView}(\text{Vec}(\widetilde{\mathcal{Y}}_n^{\mathcal{L}})))$.

The elements of the real data are usually dependent, i.e., the covariance matrix is non-diagonal in LDS. Conversely, if the elements of the data are independent, the covariance matrix is diagonal. We conduct experiments with the noise covariances in the models being diagonal and non-diagonal, respectively. In all the experiments, the LDS latent dimensionality is always set to the smallest value such that the number of parameters of LDS is greater than or equal to that of MLDS. For fair comparisons, the latent dimensionality $J$ of each LDS in the transform domain of $\mathcal{L}$-MLDS is set to the largest value such that the number of parameters of the $\mathcal{L}$-MLDS is less than or equal to that of MLDS.

\subsection{Performance Results with Real Data}

We use the following datasets in evaluations, and the codes are avaialbe online \cite{tensorlet}.

\textbf{SST} \cite{rogers2013multilinear}: A 5-by-6 grid of sea-surface temperatures from $5^{\circ}$N, $180^{\circ}$W to $5^{\circ}$S, $110^{\circ}$W recorded hourly from 7:00PM on 4/26/94 to 3:00AM on 7/19/94, yielding 2000 epochs.

\textbf{Video} \cite{rogers2013multilinear}: 1171 grayscale frames of ocean surf during low tide.

\textbf{Tesla} \cite{xueqiu}: Opening, high, low, closing and adjusted-closing of the stock prices of 14 car and oil companies (e.g., Tesla Motors Inc.), from 5/4/13 to 5/4/18 (1260 epochs).

\textbf{NASDAQ-100} \cite{qin2017dual}: Opening, closing, high, and low for 50 randomly-chosen NASDAQ-100 companies, from 7/26/16 to 4/28/17 (2186 epochs).

For the SST dataset, each model was trained on the first 1800 epochs and tested on the last 200 epochs. When the latent state dimensionality of the MLDS is set to $2\times3$, the results are shown in Figure \ref{diag}(a) and Figure \ref{full}(a). For the Video dataset, a $10\times10$ patch for each frame, each model was trained on the first 1000 frames and tested on the last 171 frames. When the latent state dimensionality of the MLDS is set to $5\times5$, the results are shown in Figure \ref{diag}(b) and Figure \ref{full}(b). For the Tesla dataset, a $14\times5$ patch for each epoch, each model is trained on the first 1100 epochs and tested on the last 160 epochs. When the latent dimensionality of the MLDS is set to $5\times2$, the result are shown in Figure \ref{diag}(c) and Figure \ref{full}(c). For the NASDAQ-100 dataset, each model is trained on the first 2000 epochs and tested on the last 186 epochs. When the latent dimensionality of the MLDS is set to a $10\times3$, the results are shown in Figure \ref{diag}(d) and Figure \ref{full}(d).

\begin{figure}[H]
\centering
\begin{minipage}[t]{0.241\linewidth}
\centering
\includegraphics[width=1.4in]{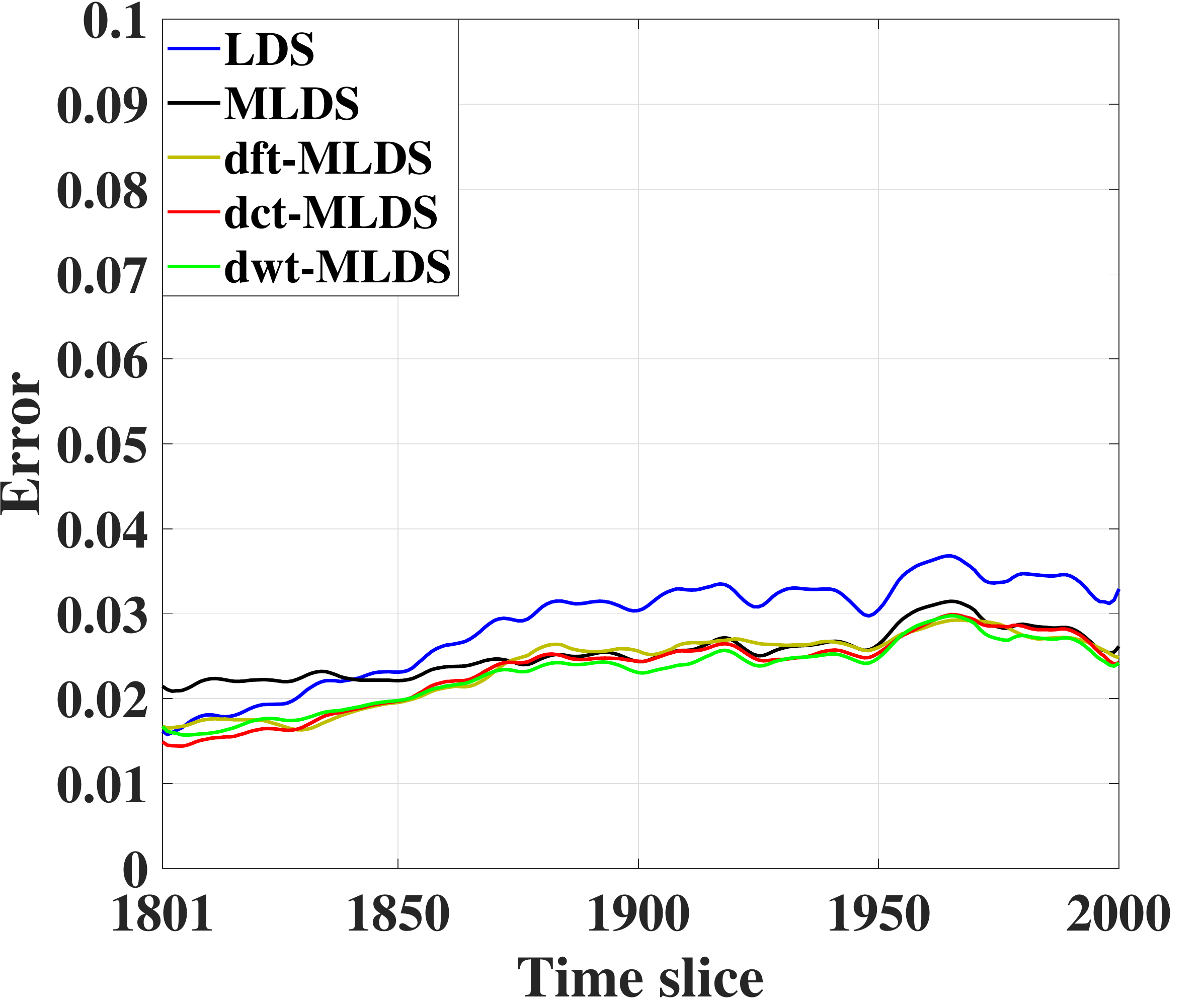}
\subfigure{(a) SST}
\end{minipage}
\begin{minipage}[t]{0.241\linewidth}
\centering
\includegraphics[width=1.4in]{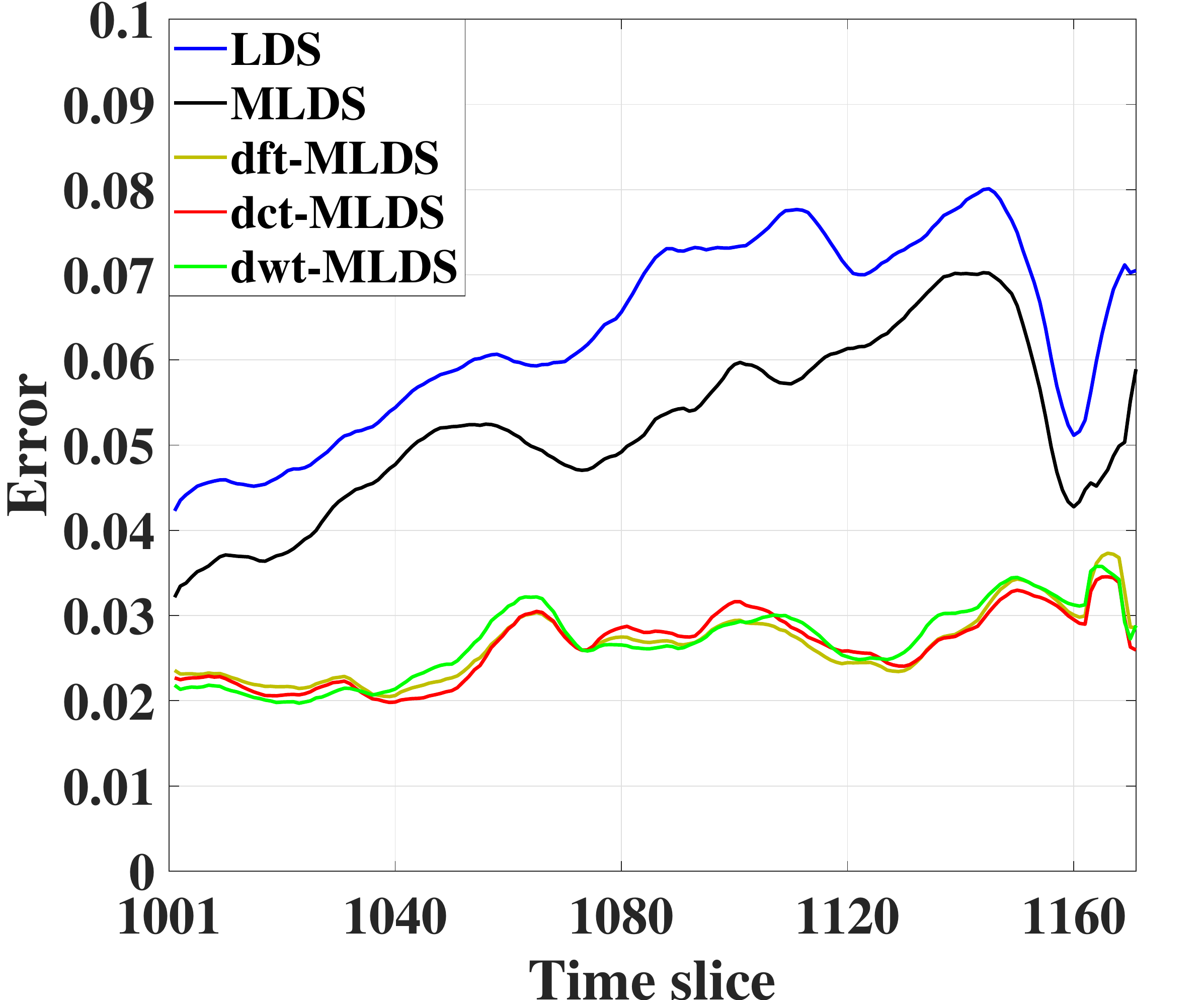}
\subfigure{(b) Video}
\end{minipage}
\begin{minipage}[t]{0.241\linewidth}
\centering
\includegraphics[width=1.4in]{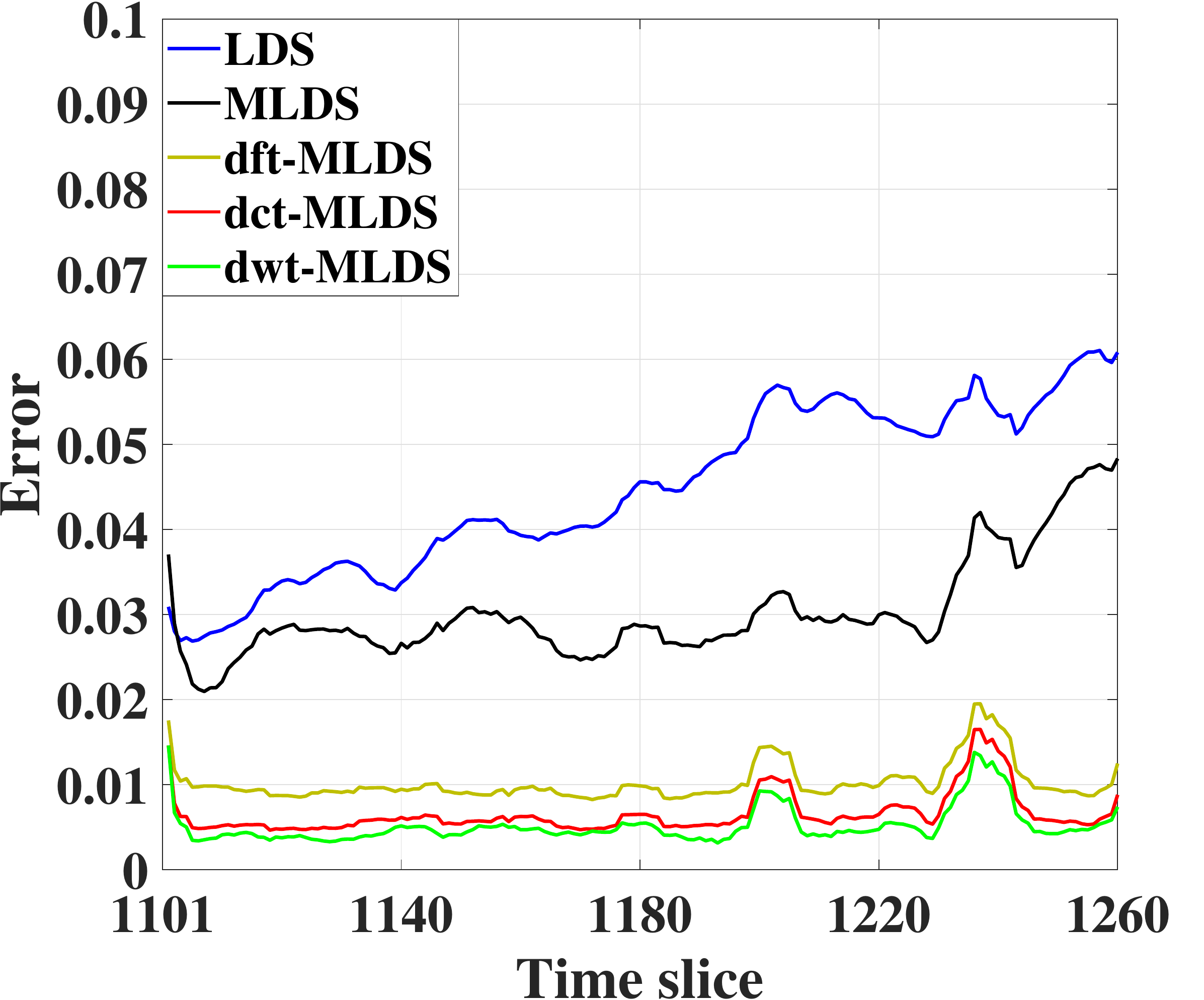}
\subfigure{(c) Tesla}
\end{minipage}
\begin{minipage}[t]{0.241\linewidth}
\centering
\includegraphics[width=1.4in]{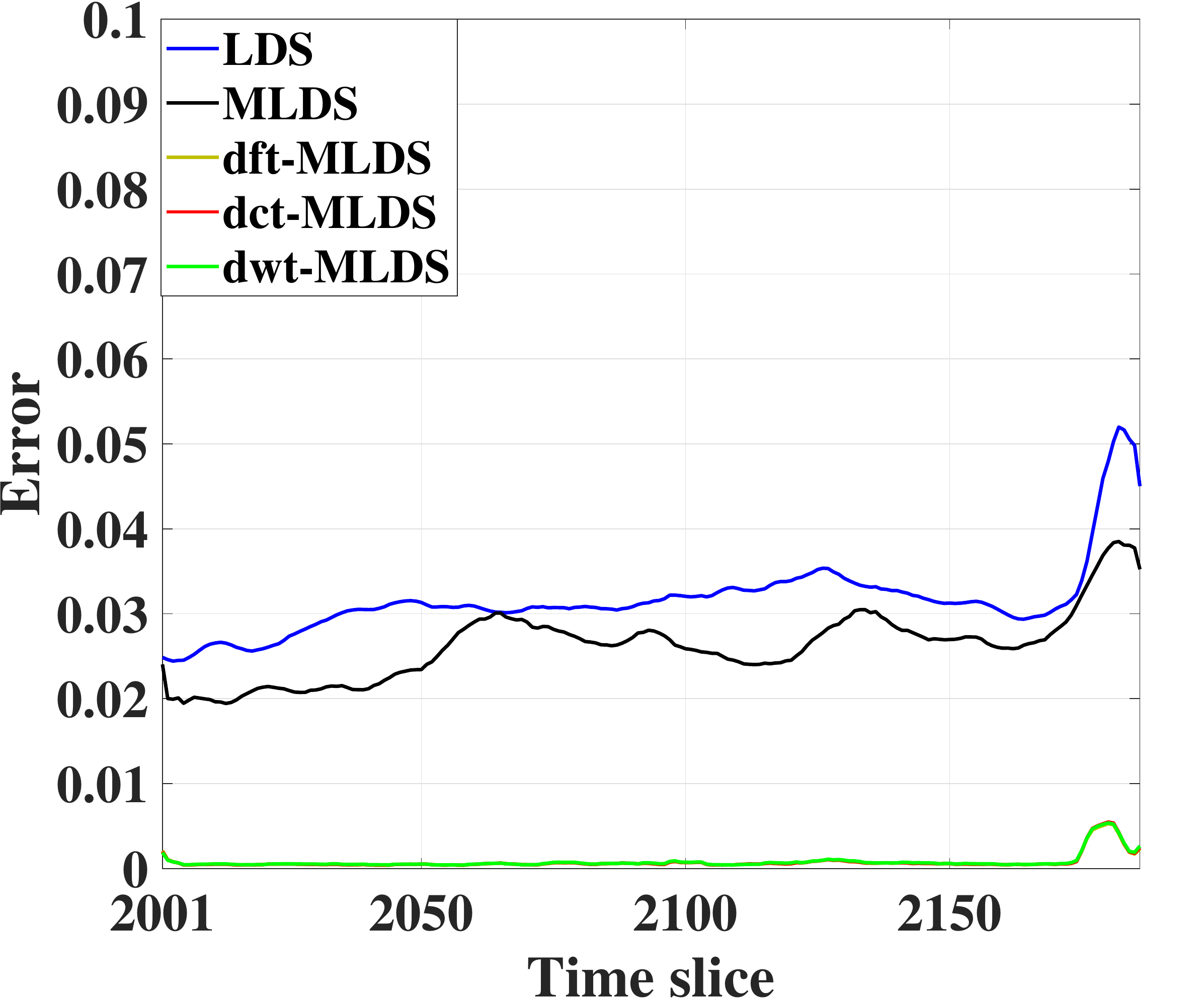}
\subfigure{(d) NASDAQ-100}
\end{minipage}
\caption{Performance results for LDS, MLDS, dct-MLDS, dft-MLDS and dwt-MLDS using real data with the covariances of the noises being diagonal.}\label{diag}
\vspace{-0.25in}
\end{figure}
\begin{figure}[H]
\centering
\begin{minipage}[t]{0.241\linewidth}
\centering
\includegraphics[width=1.4in]{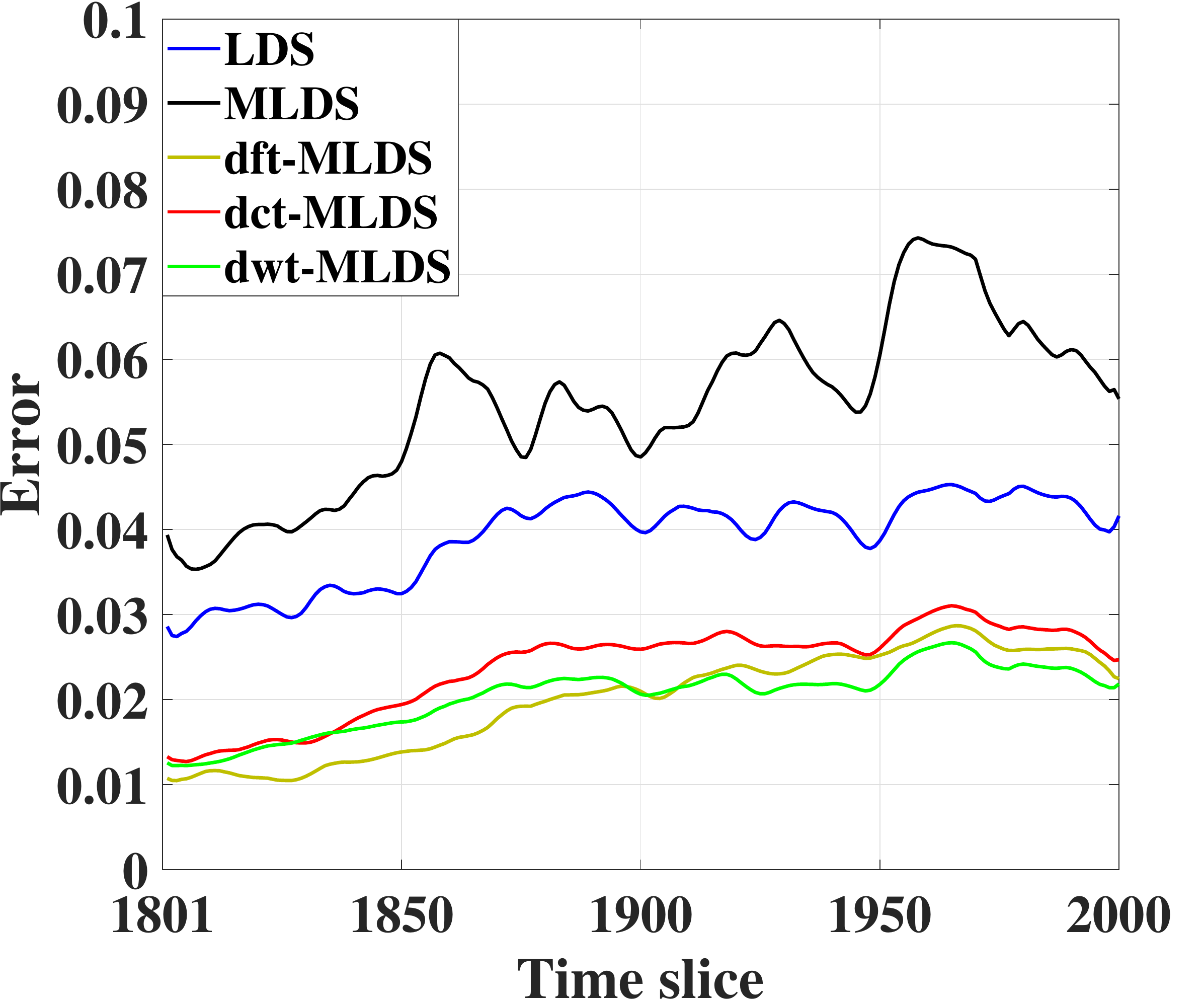}
\subfigure{(a) SST}
\end{minipage}
\begin{minipage}[t]{0.241\linewidth}
\centering
\includegraphics[width=1.4in]{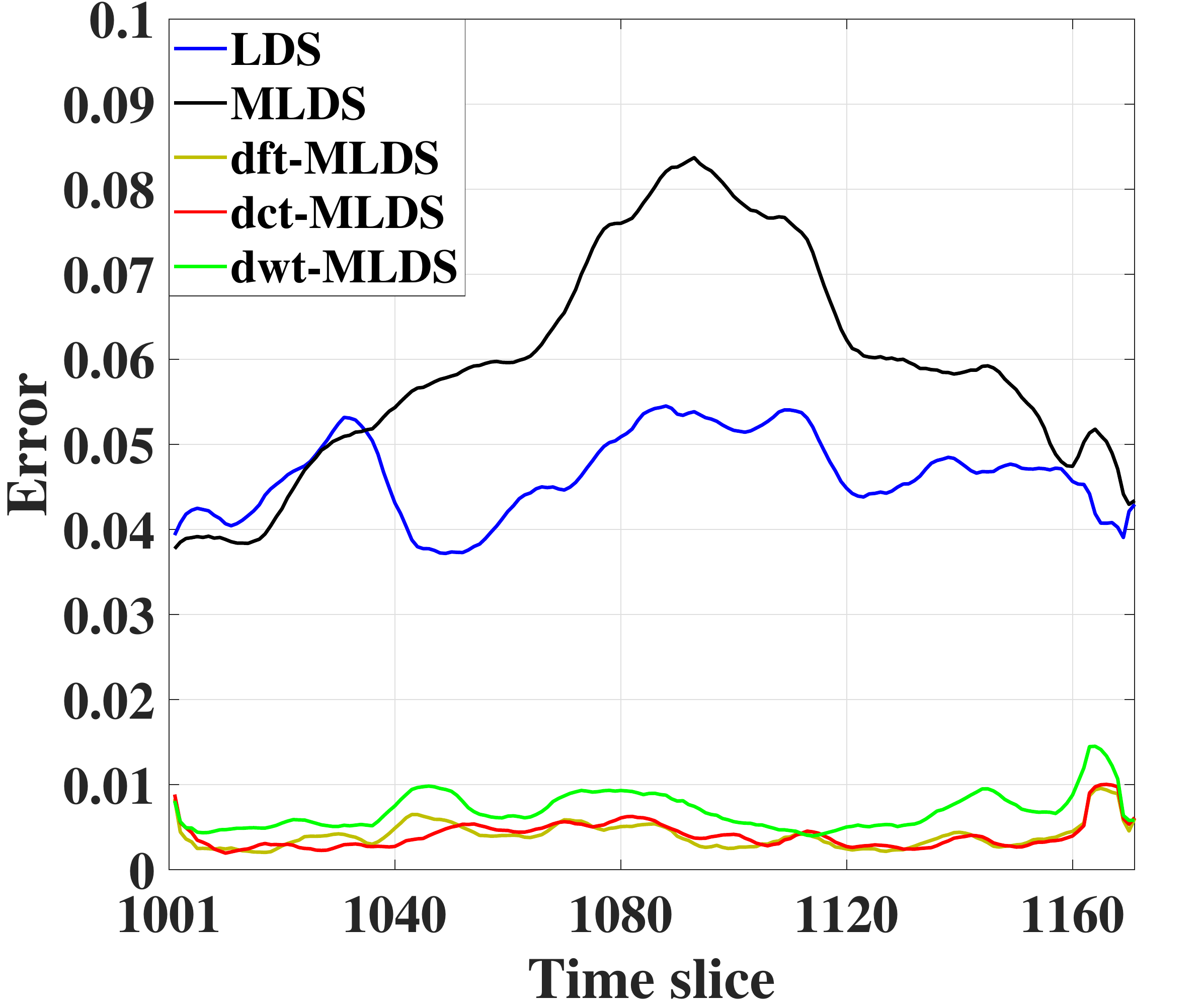}
\subfigure{(b) Video}
\end{minipage}
\begin{minipage}[t]{0.241\linewidth}
\centering
\includegraphics[width=1.4in]{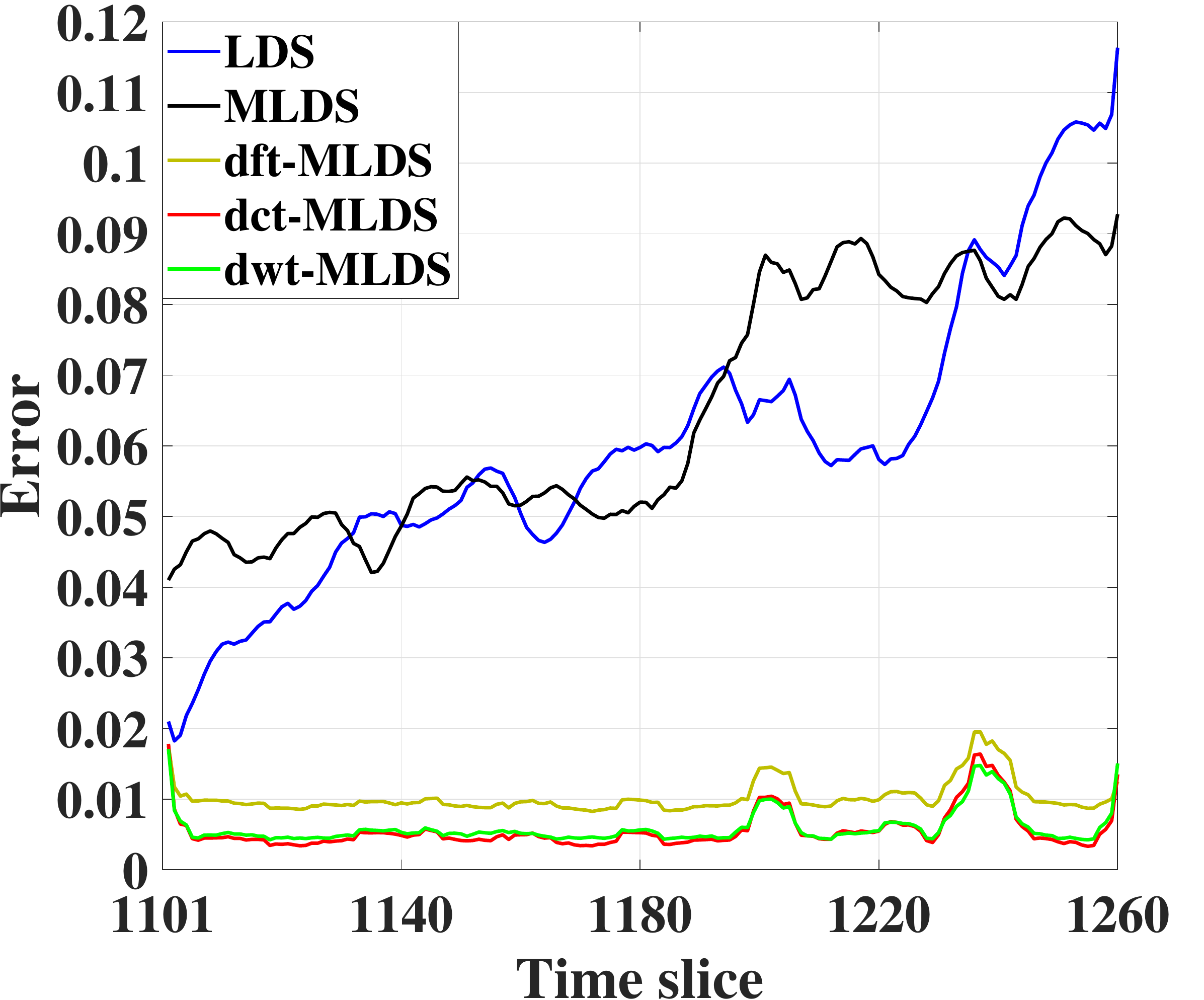}
\subfigure{(c) Tesla}
\end{minipage}
\begin{minipage}[t]{0.241\linewidth}
\centering
\includegraphics[width=1.4in]{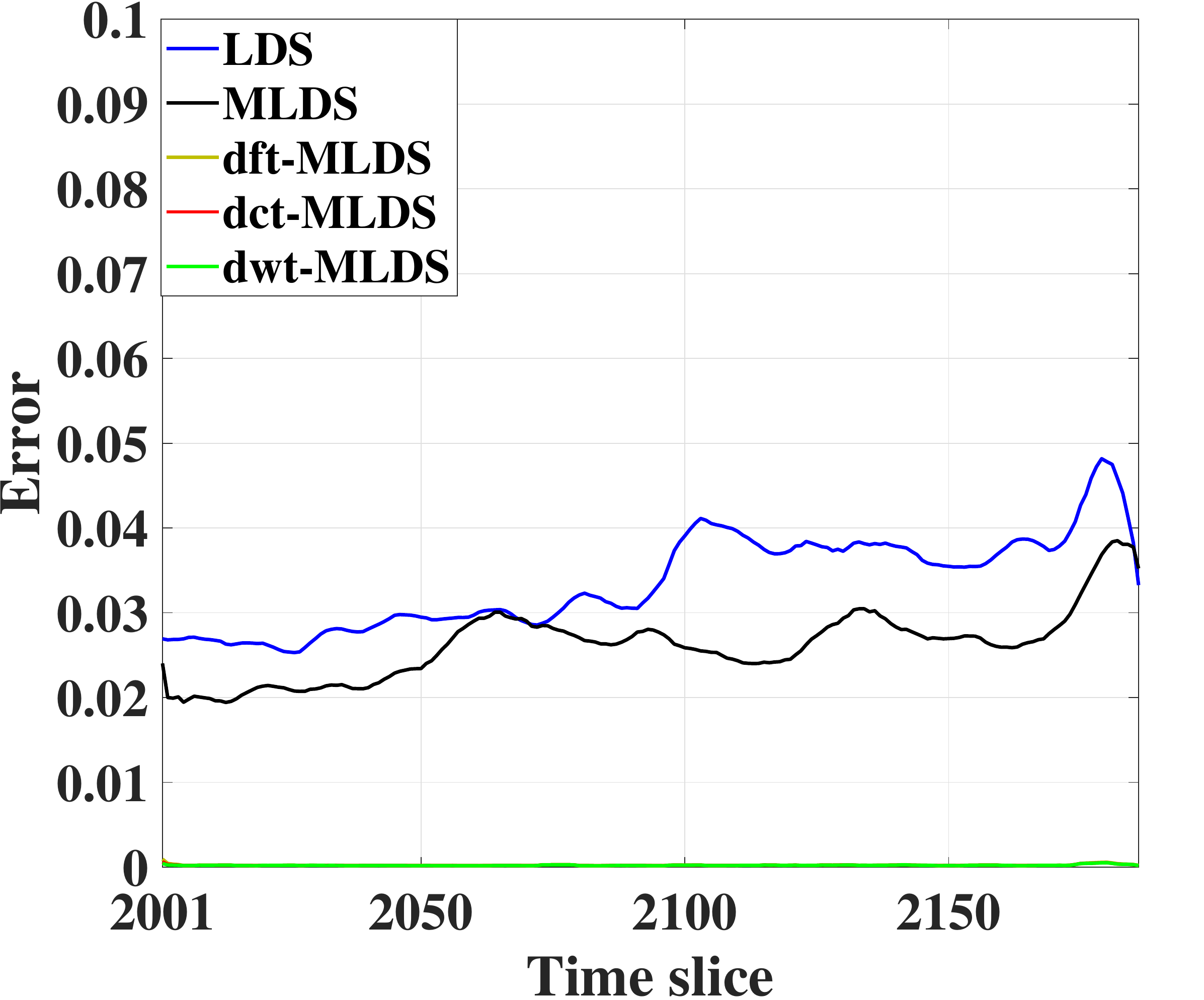}
\subfigure{(d) NASDAQ-100}
\end{minipage}
\caption{Performance results for LDS, MLDS, dct-MLDS, dft-MLDS and dwt-MLDS using real data with the covariances of noises being non-diagonal.}\label{full}
\vspace{-0.2in}
\end{figure}
\begin{figure}[H]
\centering
\begin{minipage}[t]{0.48\textwidth}
\centering
\includegraphics[width=6.6cm]{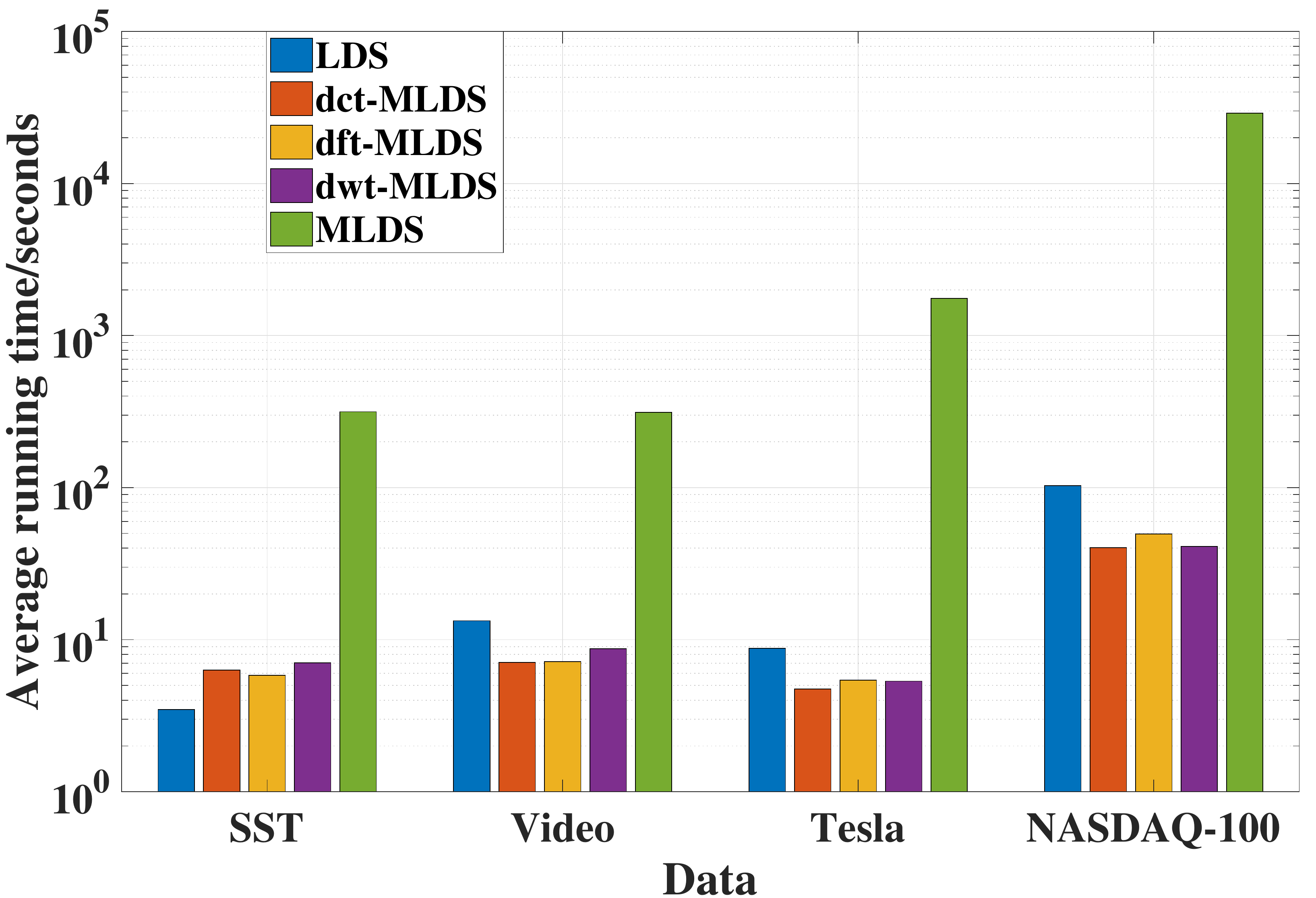}
\subfigure{(a)}
\end{minipage}
\begin{minipage}[t]{0.48\textwidth}
\centering
\includegraphics[width=6.6cm]{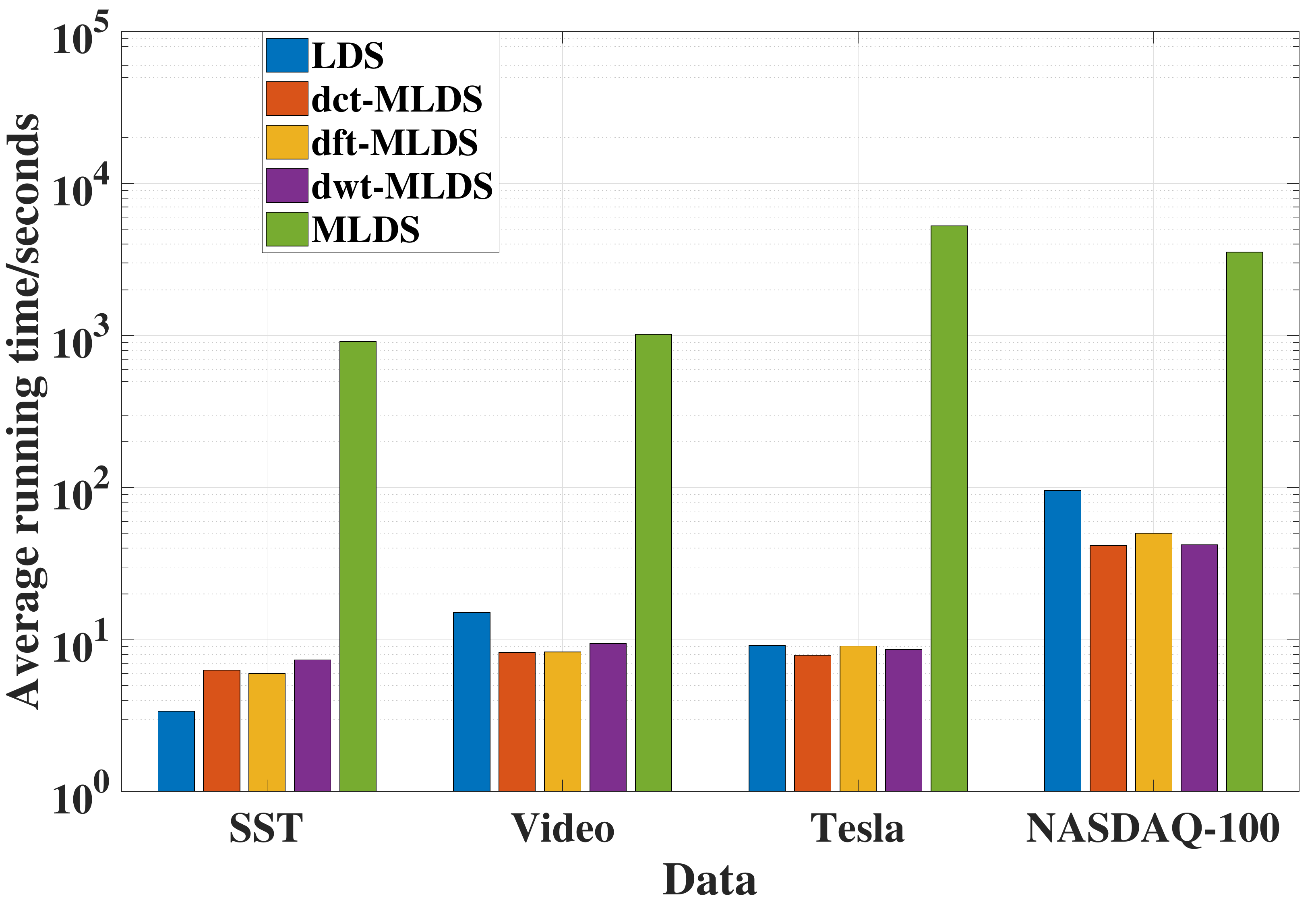}
\subfigure{(b)}
\end{minipage}
\caption{The running time for LDS, MLDS, dct-MLDS, dft-MLDS and dwt-MLDS ( the runtime of transformation is considered in each $\mathcal{L}$-MLDS). (a) corresponds to Figure \ref{diag} and (b) to Figure \ref{full}}\label{runtime}
\vspace{-0.2in}
\end{figure}
The comparisons ( shown in Figure \ref{diag} and \ref{full}) demonstrate that our $\mathcal{L}$-MLDS is able to achieve a high prediction accuracy for arbitrary noise relationships among the tensorial elements, and reduces the relative errors by $50\% \sim 99\%$. In addition to the higher prediction accuracy, $\mathcal{L}$-MLDS reduces the training time by orders of magnitude compared to MLDS, see Figure \ref{runtime}. Simultaneously, the longer the vectorized inputs are, the more obvious the improvement will be.

\section{Conclusions}\label{sec6}
In this paper, we have proposed a novel multilinear dynamical system, $\mathcal{L}$-MLDS, to model tensor time series. For $\mathcal{L}$-MLDS, we work in a transform domain, which brings a significant reduction in model complexity compared to the prior models LDS and MLDS. With nonlinear transformations applied, $\mathcal{L}$-MLDS is able to capture the nonlinear correlations among tensors of a time series, leading to more accurate prediction than assuming linear correlations. In addition, the exact separation of an $\mathcal{L}$-MLDS into several LDSs enables efficient computation and allows for parallel processing that overcomes the curse of dimensionality when dealing with big data. On four real datasets, the proposed $\mathcal{L}$-MLDS not only achieves higher prediction accuracy, but takes less time for training than MLDS and LDS. Due to its superior performance in stock price prediction, we will integrate this scheme to the deep reinforcement learning approach for stock trading \cite{liu2018NIPS}.

\section{Acknowledgement}
This work was supported by the National Natural Science Foundation of China under Grant 61671345.

\small
\bibliographystyle{unsrt}
\bibliography{workshop_2018}

\begin{thebibliography}{10}

\bibitem{rogers2013multilinear}
Mark Rogers, Lei Li, and Stuart~J Russell,
\newblock ``Multilinear dynamical systems for tensor time series,''
\newblock in {\em Advances in Neural Information Processing Systems (NIPS)},
  2013, pp. 2634--2642.

\bibitem{bishop2006pattern}
Christopher~M Bishop,
\newblock ``Pattern recognition and machine learning,''
\newblock Springer, 1st edition, 2006.

\bibitem{ghahramani1996parameter}
Zoubin Ghahramani and Geoffrey~E Hinton,
\newblock ``Parameter estimation for linear dynamical systems,''
\newblock Tech. {R}ep., CRG-TR-96-2, University of Totronto, Dept. of Computer
  Science, 1996.

\bibitem{xiong2010temporal}
Liang Xiong, Xi~Chen, Tzu-Kuo Huang, Jeff Schneider, and Jaime~G Carbonell,
\newblock ``Temporal collaborative filtering with bayesian probabilistic tensor
  factorization,''
\newblock in {\em Proceedings of International Conference on Data Mining},
  2010, pp. 211--222.

\bibitem{surana2016dynamic}
Amit Surana, Geoff Patterson, and Indika Rajapakse,
\newblock ``Dynamic tensor time series modeling and analysis,''
\newblock in {\em IEEE 55th Conference on Decision and Control (CDC)}, 2016,
  pp. 1637--1642.

\bibitem{sun2006beyond}
Jimeng Sun, Dacheng Tao, and Christos Faloutsos,
\newblock ``Beyond streams and graphs: dynamic tensor analysis,''
\newblock in {\em Proceedings of International Conference on Knowledge
  Discovery and Data Mining}. IEEE, 2006, pp. 374--383.

\bibitem{kolda2009tensor}
Tamara~G Kolda and Brett~W Bader,
\newblock ``Tensor decompositions and applications,''
\newblock pp. 455--500, 2009.

\bibitem{liu2017fourth}
Xiao-Yang Liu and Xiaodong Wang,
\newblock ``Fourth-order tensors with multidimensional discrete transforms,''
\newblock {\em arXiv preprint arXiv:1705.01576}, pp. 1--37, 2017.

\bibitem{kilmer2013third}
Misha~E Kilmer, Karen Braman, Ning Hao, and Randy~C Hoover,
\newblock ``Third-order tensors as operators on matrices: A theoretical and
  computational framework with applications in imaging,''
\newblock {\em Journal on Matrix Analysis and Applications}, pp. 148--172,
  2013.

\bibitem{XiaoYang2016Low}
Xiao-Yang Liu, Shuchin Aeron, Vaneet Aggarwal, and Xiaodong Wang,
\newblock ``Low-tubal-rank tensor completion using alternating minimization,''
\newblock in {\em (Major revision) IEEE Transactions on Information Theory},
  2018.

\bibitem{basser2002normal}
Peter~J Basser and Sinisa Pajevic,
\newblock ``A normal distribution for tensor-valued random variables to analyze
  diffusion tensor mri data,''
\newblock in {\em Proceedings of International Symposium on Biomedical
  Imaging}. IEEE, 2002, pp. 927--930.

\bibitem{dempster1977maximum}
Arthur~P Dempster, Nan~M Laird, and Donald~B Rubin,
\newblock ``Maximum likelihood from incomplete data via the em algorithm,''
\newblock {\em Journal of the Royal Statistical Society. Series B
  (methodological)}, pp. 1--38, 1977.

\bibitem{tensorlet}
``Our codes,''
\newblock {\em http://www.tensorlet.com/}.

\bibitem{xueqiu}
``Tesla stock prices data,''
\newblock {\em https://xueqiu.com/}.

\bibitem{qin2017dual}
Yao Qin, Dongjin Song, Haifeng Chen, Wei Cheng, Guofei Jiang, and Garrison
  Cottrell,
\newblock ``A dual-stage attention-based recurrent neural network for time
  series prediction,''
\newblock {\em arXiv preprint arXiv:1704.02971}, 2017.

\end{thebibliography}

\end{document}